\def\year{2019}\relax
\pdfoutput=1
\documentclass[letterpaper]{article} 
\usepackage{aaai19}  
\usepackage{times}  
\usepackage{helvet}  
\usepackage{courier}  
\usepackage{url}  
\usepackage{graphicx}  
\usepackage[utf8]{inputenc} 
\usepackage[T1]{fontenc}    
\usepackage{url}            
\usepackage{booktabs}       
\usepackage{amsfonts}       
\usepackage{nicefrac}       
\usepackage{microtype}      

\usepackage{wrapfig}
\usepackage{xcolor}

\usepackage{amsmath,amsthm,amssymb,amsfonts}
\usepackage{siunitx}
\usepackage{algorithm}
\usepackage{algpseudocode}
\usepackage{subcaption}
\usepackage{mathtools}
\usepackage{hhline}
\usepackage{dsfont}

\usepackage{float}
\newfloat{algorithm}{t}{lop}

\newtheorem{theorem}{Theorem}
\newtheorem{definition}{Definition}

\newtheorem{lemma}[theorem]{Lemma}
\newtheorem{assumption}{Assumption}

\newcommand{\E}{\mathbb{E}}

\newcommand{\BR}{\mathrm{BR}}
\newcommand{\BReps}{\mathrm{BR}_\epsilon}
\newcommand{\badv}{b_\mathrm{adv}}
\newcommand{\badvj}{b_{\mathrm{adv},j}}
\newcommand{\bprior}{b_\textrm{prior}}
\newcommand{\Vbr}{V_\mathrm{br}}
\newcommand{\Vmodel}{\hat{V}_{\pi}}
\newcommand{\Vstark}{V^*_{b_k}}
\newcommand{\Qha}{Q_{ha}}

\DeclareMathOperator*{\argmax}{arg\,max}

\begin{document}
%
\title{Robust and Adaptive Planning under Model Uncertainty}
\author{
Apoorva Sharma$^1$, James Harrison$^2$, Matthew Tsao$^3$, Marco Pavone$^1$\\
Department of $^1$Aeronautics and Astronautics, $^2$Mechanical Engineering, $^3$Electrical Engineering\\
\texttt{\{apoorva,jharrison,mwtsao,pavone\}@stanford.edu}\\
Stanford University, Stanford, CA 94305
}
\maketitle

\newcommand\jhtodo[1]{\textcolor{red}{[JH: #1]}} 
\newcommand{\jhmargin}[2]{{\color{red}#1}\marginpar{\color{red}\raggedright\footnotesize [JH]:#2}}

\newcommand\astodo[1]{\textcolor{blue}{[AS: #1]}} 
\newcommand{\asmargin}[2]{{\color{blue}#1}\marginpar{\color{blue}\raggedright\footnotesize [AS]:#2}}

\newcommand\mptodo[1]{\textcolor{green}{[MP: #1]}} 
\newcommand{\mpmargin}[2]{{\color{green}#1}\marginpar{\color{green}\raggedright\footnotesize [MP]:#2}}




\newcommand{\algName}{RAMCP}
\newcommand{\algNameLong}{Robust Adaptive Monte Carlo Planning}

\newcommand{\algNameFull}{\algName{}-F}
\newcommand{\algNameInc}{\algName{}-I}

\begin{abstract}
  Planning under model uncertainty is a fundamental problem across many applications of decision making and learning.
  In this paper, we propose the \algNameLong{} (\algName{}) algorithm, which allows computation of risk-sensitive Bayes-adaptive policies that optimally trade off exploration, exploitation, and robustness.
  \algName{} formulates the risk-sensitive planning problem as a two-player zero-sum game, in which an adversary perturbs the agent's belief over the models.
  We introduce two versions of the RAMCP algorithm. The first, \algNameFull{}, converges to an optimal risk-sensitive policy without having to rebuild the search tree as the underlying belief over models is perturbed. The second version, \algNameInc{}, improves computational efficiency at the cost of losing theoretical guarantees, but is shown to yield empirical results comparable to \algNameFull{}.
   \algName{} is demonstrated on an $n$-pull multi-armed bandit problem, as well as a patient treatment scenario.
\end{abstract}

\section{Introduction}
In many sequential decision making domains, from personalized medicine to human-robot interaction, the underlying dynamics are well understood save for some latent parameters, which might vary between episodes of interaction. For example, we might have models for the evolution of a disease of a patient under various treatments, but they may depend on unobserved, patient-specific physiological parameters. Faced with a new patient, the agent must learn over the course of a \textit{single} episode of interaction, as it simultaneously tries to identify the underlying parameters while maximizing its objective of improving the patient's health.

Such challenges are commonplace, yet not well addressed by standard episodic reinforcement learning, which assumes no prior knowledge and learns over the course of repeated interaction on the same system. These problems are better addressed by a Bayesian approach to reinforcement learning, in which the agent can leverage prior knowledge in the form of a \textit{belief distribution} over a family of likely models, which can be updated via Bayes' rule as the agent interacts with the system \cite{ghavamzadeh2015bayesian}. Incorporating this prior knowledge enables effective learning within a single episode of interaction, and also allows the the agent to consider how to balance identification of the latent parameters with maximizing the objective.

While leveraging a prior distribution over models is powerful, coming up with accurate priors remains a challenge. In the contexts of human interaction scenarios like patient treatment, these distributions might be obtained experimentally from past interactions, or may be chosen heuristically by a domain expert. Therefore, this prior over models is likely to be inaccurate, and thus it is paramount that the agent plan in a manner that is robust to incorrect priors, especially in safety-critical settings.

\paragraph{Contributions}
The contributions of this work are three-fold. First, we present (to our knowledge) the first mathematical framework to incorporate robustness to priors in the context of Bayesian RL. Second, we present \algNameLong{} (\algName), an sampling-based tree search algorithm for online planning in this framework for discrete MDPs and discrete priors over models. The approach fundamentally consists of an adversarial weighting step on top of standard risk-neutral tree-search approaches to Bayesian RL such as BAMCP \cite{guez2013scalable}. In particular, we introduce two version of the \algName{} algorithm, which we refer to as \algNameFull{} and \algNameInc{}. For the first, we prove that this adversarial optimization does not oscillate, and indeed converges to the optimal solution. \algNameInc{}, on the other hand, sacrifices convergence guarantees for empirical performance. Finally, we demonstrate the algorithms through numerical experiments including a patient treatment scenario, and compare the performance of the two versions of the \algName{} algorithm.

\section{Background}
\label{sec:background}
This work considers adding robustness to model-based Bayesian reinforcement learning. In this setting, we wish to control an agent in a system defined by the Markov Decision Process (MDP) $\mathcal{M} = (\mathcal{S},\mathcal{A},T,R,H)$, where $\mathcal{S}$ is the state space, $\mathcal{A}$ is the action space, $T(s'|s,a)$ is the transition function, $R(s,a,s')$ is the stage-wise reward function, and $H$ is the problem horizon. We assume that the exact transition dynamics $T$ depend on a parameter $\theta$, and denote this dependence as $T_\theta$. We assume that the agent knows this MDP, but is uncertain about the true setting of the parameter $\theta$, and instead maintains a belief distribution over this parameter. This structured representation of the agent's knowledge (and lack of knowledge) of the planning problem allows considering the objectives of exploration and robustness in addition to the standard MDP objective of maximizing reward.

\paragraph{The Explore/Exploit Dilemma}
As the agent acts in the true MDP, the observed state transitions will provide information about the true underlying system parameters. An agent might \textit{explore}, i.e. choose actions with the aim of reducing uncertainty over $\theta$, or \textit{exploit}, choosing actions to maximize cumulative reward given its current estimates of $\theta$.
The Bayesian setting allows optimally making this tradeoff.

Writing the observed transitions so far, or the \textit{history} in an environment at time $t$ as $h_t = (s_0, a_0, \ldots, s_t)$ and given a prior distribution over the model parameters $\bprior$, we can define optimal behavior in the Bayesian setting. Let $\mathcal{H}$ denote the set of possible histories for a given MDP. Then, we will write the set of stochastic history-dependent polices $\pi : \mathcal{H} \times \mathcal{A} \to [0,1]$ as $\Pi$. Let
\begin{align*}
    V(h_t,&\pi) =\\
    &\E_{\pi,b} \left[ \sum^{H-1}_{\tau=t} R(s_\tau,a_\tau,s_{\tau+1}) \mid ( s_0, a_0, \ldots, s_t ) = h_t \right]
\end{align*}
 denote the value function associated with policy $\pi$, for history $h$, and with model distribution $b$. A history-dependent policy $\pi^*$ is said to be Bayes-optimal with respect to the prior $b_\textrm{prior}$ if it has associated value function $V(\{s\},\pi^*) = \sup_{\pi \in \Pi} V(\{s\}, \pi)$ \cite{martin1967bayesian}.

The problem is a special case of a Partially Observable Markov Decision Processes (POMDP), where the hidden portion of the state (here, the parameters $\theta$) is fixed over the course of an episode. As with all POMDPs, this problem can be cast into a belief-state MDP by augmenting the state at time $t$ with the posterior belief (or, equivalently, the history up to that point $h_t$). This is the Bayes-Adpative Markov Decision Process (BAMDP) formulation, and the optimal policy in this MDP is the Bayes-optimal policy \cite{duff2002optimal}.
In general, optimizing these policies is computationally difficult. Information-state techniques (as in, e.g., Gittins indices for bandit problems \cite{gittins2011multi}) are typically intractable due to a continuously growing information-state space \cite{duff2002optimal}. Offline global value approximation approaches, typically based on offline POMDP solution methods, scale poorly to large state spaces \cite{ghavamzadeh2015bayesian}. Online approaches \cite{wang2005bayesian}, \cite{guez2013scalable}, \cite{chen2016pomdp} use tree search with heuristics to either simplify the problem or guide the search.

\paragraph{Robustness to Incorrect Priors}
While encoding model uncertainty through a prior over model parameters enables optimally balancing exploration and exploitation, it is likely that these priors may be inaccurate.
Previous work in policy optimization for BAMDPs has focused on optimizing performance in expectation, and thus does not offer any notion of robustness to misspecified priors \cite{guez2013scalable}. Robust MDPs are posed as MDPs with uncertainty sets over state transitions, and approaches to this problem aim to optimize the worst-case performance over all possible transition models \cite{nilim2005robust}. However, this minimax approach does not consider the belief associated with a transition model, and thus is typically over-conservative and can not optimally balance exploration and exploitation.

Tools from risk theory can be used to achieve tunable, distribution-dependent conservatism. Given a reward random variable $Z$, a \textit{risk metric} is a function $\rho(Z)$ that maps the uncertain reward to a real scalar, which encodes a preference model over uncertain outcomes where higher values of $\rho(Z)$ are preferred. A key concept in risk theory is that of a \textit{coherent} risk metric. These metrics satisfy axioms originally proposed in \cite{artzner1999coherent}, which ensure a notion of rationality in risk assessment. We refer the reader to \cite{majumdar2017should} for a more thorough discussion of why coherent risk metrics are a useful tool in decision making. While expectation and worst-case are two possible coherent risk metrics, the set of CRMs is a rich class of metrics including the Conditional Value at Risk (CVaR) metric popular in mathematical finance \cite{majumdar2017should}. The rationality of coherent risk metrics contrasts with standard approaches (especially in stochastic control \cite{glover1988state}) such as mean-variance and exponential risk metrics, for which there are simple examples of clearly absurd decision-making \cite{rabin2001anomalies}.

Applying risk metrics to the reward of a single MDP, yields an optimization problem with strong connections to the Robust MDP formulation, with the choice of risk metric in the risk-sensitive MDP corresponding to a particular choice of uncertainty set in the Robust MDP formulation \cite{chow2015risk} \cite{osogami2012robustness}. In contrast, in this work we apply tools from risk theory to the Bayesian setting in which we have a prior over MDPs to balance all three objectives of robustness, exploration, and exploitation in a coherent mathematical framework. To our knowledge, this work represents the first approach that considers all three of these objectives in the context of sequential decision making.

\section{Problem Statement}
\label{sec:problem}

We aim to compute a history dependent policy $\pi$ which is optimal according to a coherent risk metric over model uncertainty. To make this objective more concrete, let $\tau = (s_0, a_0, \dots, s_{H-1}, a_{H-1}, s_{H})$ be particular trajectory realization. Note that the probability of a given trajectory depends on both the choice of policy $\pi$ and the transition dynamics of the MDP, $T_\theta$.
The cumulative reward of a given trajectory $J(\tau)$ can be calculated by summing the stage-wise rewards
\begin{equation}
  J(\tau) = \sum_{t=0}^{H-1} R(s_t, a_t, s_{t+1}).
\end{equation}
Note that since the distribution over $\tau$ is governed by the stochasticity in each environment as well as the uncertainty over environments, the distribution of the total cost of a trajectory $J(\tau)$ is as well.

In this work, we focus our attention to risk-sensitivity with respect to the randomness from model uncertainty only. Concretely, we can write the objective as
\begin{align}
  \begin{split}
    \label{eqn:opt_metric}
    \Pi^* = \argmax_\pi  &~~ \rho \left( \E \left[ J(\tau) | \theta, \pi \right] \right),
  \end{split}
\end{align}
where the risk metric is with respect to the random variable induced by the distribution over models. Note that $\Pi^*$ denotes the set of optimal policies.

In this work, we consider a finite collection of $M$ possible parameter settings, $\Theta = \{\theta_i\}_{i=1}^{M}$, and write our prior belief over $\Theta$ as the vector $\bprior$. While limiting ourselves to discrete distributions over model parameters is somewhat restrictive, these simplifications are common in, for example, sequential Monte Carlo. Indeed, computing continuous posterior distributions exactly is often intractable, making this discrete approximation necessary \cite{guez2014bayes}.

The above objective differs from that typically used in the risk-sensitive reinforcement learning literature, which applies the risk metric $\rho$ directly to the total reward random variable $J(\tau)$ \cite{tamar2017sequential}. In contrast, we first marginalize out the effects of stochasticity via the expectation, and then apply $\rho$ to the multinomial random variable $\E \left[ J(\tau) | \theta = \theta_i \right]$ where each outcome corresponds to the expected value of the policy $\pi$ on model $\theta_i$. The robustness provided by risk-sensitivity protects against modeling errors in distribution. In the BAMDP context, the primary error in distribution is in the model belief: the true distribution over models has all its mass on one model, so effectively any belief will be incorrect. Furthermore, these beliefs will be updated online, and thus are susceptible to noise. Comparatively, we assume that for a particular model, the stochasticity in transition dynamics is well characterized. Thus, we argue that the optimization objective (\ref{eqn:opt_metric}) is well-aligned with our goal of enabling robustness to model uncertainty.

\section{Approach}

In this section we discuss the high-level approach taken in \algName{}. We begin by reformulating (\ref{eqn:opt_metric}) as a two-player, zero-sum game. This game is played between an agent computing optimal history-dependent policies with respect to a belief over models, and an adversary perturbing this belief. Armed with this problem reformulation, we describe Generalized Weakened Fictitious Play (GWFP) \cite{leslie2006generalised}, a framework for computing Nash equilibria of two-player zero-sum games (as well as a collection of other game settings). Finally, we outline application of GWFP to (\ref{eqn:opt_metric}).

\subsection{Reformulation as a Zero-Sum Game}

Our reformulation of the objective stems from a universal representation theorem which all coherent risk metrics (CRMs) satisfy.
\begin{theorem}[Representation Theorem for Coherent Risk Metrics \cite{artzner1999coherent}] Let $(\Omega, \mathcal{F}, \mathbb{P})$ be a probability space, where $\Omega$ is a finite set with cardinality $|\Omega|$, $\mathcal{F}$ is a $\sigma-$algebra over subsets (i.e., $\mathcal{F}=2^\Omega$), probabilities are assigned according to $\mathbb{P} = (p(1), \dots, p(|\Omega|))$, and $\mathcal{Z}$ is the space of reward random variables on $\Omega$. Denote by $\mathcal{C}$ the set of valid probability densities:
 \begin{equation}
   \mathcal{C} := \left\{ \zeta \in \mathbb{R}^{|\Omega|} \mid \sum_{i=1}^{|\Omega|} p(i) \zeta(i) = 1, \zeta \ge 0 \right\}.
 \end{equation}
Define $p_\zeta \in \mathbb{R}^{|\Omega|}$ as $p_{\zeta}(i) = p(i) \zeta(i), i = 1, \dots, |\Omega|$. A risk metric $\rho: \mathcal{Z} \to \mathbb{R}$ with respect to the space $(\Omega, \mathcal{F}, \mathbb{P})$ is a coherent risk metric if and only if there exists a compact convex set $\mathcal{B} \subset \mathcal{C}$ such that for any $Z \in \mathcal{Z}$:
\begin{equation}
  \label{eqn:crm_rep}
  \rho(Z) = \min_{\zeta \in \mathcal{B}} \mathbb{E}_{p_\zeta}[Z] = \min_{\zeta \in \mathcal{B}} \sum_{i=1}^{|\Omega|} p(i) \zeta(i) Z(i).
\end{equation}
\end{theorem}
This theorem offers an interpretation of CRMs as a worst-case expectation over a set of densities $\mathcal{B}$, often referred to as the \textit{risk envelope}.
The particular risk envelope depends on the risk metric chosen, as well as the degree of risk-aversity (which is captured by a parameter for most coherent risk metrics).
In this work we focus on \textit{polytopic risk metrics}, for which the envelope $\mathcal{B}$ is a polytope. For this class of metrics, the constraints on the maximization in Equation \ref{eqn:crm_rep} become linear in the optimization variable $\zeta$, and thus solving for the value of the risk metric becomes a tractable linear programming problem. Note that solving this linear program is at the core of \algName{}. This computational step is why \algName{} is restricted to discrete distributions over models, and can not directly be extended to continuous beliefs. Polytopic risk metrics constitute a broad class of risk metrics, encompassing risk neutrality, mean absolute semi-deviation, spectral risk measures, as well as the CVaR metric often used in financial applications, with the choice of metric determining the form of the polytope \cite{eichhorn2005polyhedral}.

Through the representation theorem for coherent risk metrics (Equation \ref{eqn:crm_rep}), we can understand Equation \ref{eqn:opt_metric} as applying an adversarial reweighting $\zeta$ to the distribution over models $b_\mathrm{prior}$. Let $b_{\mathrm{adv}}(i) = b_\mathrm{prior}(i)\zeta(i), ~i=1,\dots,M$ represent this reweighted distribution. Thus, (\ref{eqn:opt_metric}) may be written
\begin{align}
  \begin{split}
    \label{eqn:opt_objective}
    \Pi^* = \argmax_\pi &~ \min_{\zeta \in \mathcal{B}} \mathbb{E}_{\theta \sim b_{\mathrm{adv}}} \left[ \mathbb{E} \left[ J(\tau) \mid \theta, \pi \right] \right],
  \end{split}
\end{align}
where again $\Pi^*$ denotes the set of optimal policies. We are interested simply in finding a single policy within this set, as opposed to the full set. Note that this takes the form of a two player zero-sum game between the agent (the maximizer) and an adversary (the minimizer). One play of this game corresponds to the following three step sequence:
\begin{enumerate}
  \item The adversary acts according to its strategy, choosing $b_\mathrm{adv}$ from the risk envelope.
  \item Chance chooses $\theta \sim b_\mathrm{adv}$.
  \item The agent acts according to its strategy, or policy, $\pi(h)$ in the MDP with dynamics $T_\theta$.
\end{enumerate}
The action of the adversary is to choose a perturbation to the belief distribution from the polytope of valid disturbances that minimizes the expected performance of the agent. The agent seeks to compute an optimal policy under this perturbed belief. The solution to Equation \ref{eqn:opt_objective} is therefore the optimal Nash equilibrium of the two player game, which we denote as $(b_\mathrm{adv}^*, \pi^*)$.

\subsection{Computing Nash Equilibria}

Having formulated (\ref{eqn:opt_metric}) as a two-player zero-sum game, we leverage tools developed in algorithmic game theory to efficiently compute Nash equilibria. For two-player, zero-sum games, a Nash equilibrium can be directly computed by solving a linear program of size proportional to the strategy space of each player. In the context of our problem statement, the agent's strategy space is the space of all history dependent policies, and thus solving for a Nash equilibrium directly is computationally intractable. To compute the Nash equilibria of the game defined by (\ref{eqn:opt_metric}), we apply iterative techniques which converge to equilibria over repeated simulations of the game.

Fictitious Play \cite{brown1949some} is a process in which players repeatedly play a game and update their strategies toward the best-response to the average strategy of their opponents. This process has been shown to asymptotically converge to a Nash equilibrium. In \cite{leslie2006generalised}, the authors introduced Generalized Weakened Fictitious Play (GWFP), which allows for computation of approximate best-responses but maintains convergence guarantees, and thus has worked well for large-scale extensive form games \cite{heinrich2015fictitious}. GWFP converges to Nash equilibria in several classes of games, including two-player zero-sum games such as (\ref{eqn:opt_objective}).
For the risk-sensitive BAMDP, the GWFP updates to the adversary strategy $b$ and agent strategy $\pi$ are:
\begin{align}
    b_{k+1} &= (1 - \alpha_{k+1}) b_k + \alpha_{k+1} \BReps ( \pi_k ), \label{eq:b_update} \\
    \pi_{k+1} &= (1 - \alpha_{k+1}) \pi_k + \alpha_{k+1} \BReps ( b_k ) 
\end{align}
where $\BReps(\sigma)$ represents an $\epsilon$-suboptimal best response\footnote{If an opposing player chooses strategy $\sigma$, then an $\epsilon$-suboptimal best response to $\sigma$ is a strategy such that the player obtains a payoff (or cumulative reward) within $\epsilon$ of that of an optimal response (which is itself referred to as a best response).} to strategy $\sigma$, and $\alpha_{k+1}$ is an update coefficient chosen such that $\sum_{k=1}^\infty \alpha_k = \infty$ and $\lim_{k \to \infty} \alpha_k = 0$.
 While $\BReps(\sigma)$ is typically used to refer to the set of $\epsilon$-best responses, we will use this to refer to a strategy within this set. In this work, we set $\alpha_k = 1/k$ and thus both strategies represent running averages of the best-responses, a property we leverage in our algorithm. Initial values of the belief and policy may be chosen arbitrarily.

The adversarial best response $\BReps(\pi_k)$ can be computed by solving the linear program
\begin{equation}
\label{eqn:LP}
\begin{split}
  &\min_{b,\zeta \in \mathcal{B}} \sum_{i=1}^{M} \hat{V}_{\pi_k}(i) b(i)\\ &\,\,\textrm{s.t.}\,\, b_\textrm{prior}(i) \zeta(i)  = b(i), \,\,\,i = 1, \ldots, M,
\end{split}
\end{equation}

where $\mathcal{B}$ is the polytopic risk envelope, and $\hat{V}_{\pi_k}(i)$ is an estimate of $V_{\pi_k}(i) := \E \left[ J(\tau) \mid \theta = \theta_i, \pi = \pi_k \right]$. The suboptimality of the solution of this LP is bounded by the error in $\hat{V}_{\pi_k}(i)$.

The agent's best response $\BReps(b_k)$ is a history dependent policy  $ \pi(h) = \argmax_a \hat{Q}_{b_k}(h,a)$, where $\hat{Q}_{b_k}(h,a)$ is an estimator of $Q^*_{b_k}(h,a)$, the value of taking action $a$ at history $h$, then acting optimally when $\theta$ is drawn from $b_k$ at the start of the episode. However, computation of this policy is non-trivial, and a na\"ive approach to value estimation would involve substantial repeated computation that would result in poor performance.

\begin{algorithm}[t]
\caption{\label{alg:ramcp} \algName{}}
\centering
\begin{algorithmic}[1]
\small
  \Function{Search}{$s_0$, $b_\mathrm{prior}$}
    \State $\Vmodel(i) \gets 0$ for all $i = 1,\dots,M$
    \State $k \gets 0$
    \State $\badv \gets b_\mathrm{prior}$
    \While{within computational budget}
        \State $k \gets k + 1$
        \State $w \gets M \cdot \BReps(\pi)$
        \For{$i=1$ \textbf{to} $M$}
            \State $w \gets M \cdot \badv(i)$
            \State $\Vbr \gets $ \Call{Simulate}{$s_0$, $\theta_i$, $w$}
            \label{alg:rollout_value_estimation}
            \State $\Vmodel(i) \gets \Vmodel(i) + \frac{1}{k} ( \Vbr - \Vmodel(i))$ \label{alg:model_value_update}
        \EndFor
        \State {\color{blue} \Call{ComputeQValues}{$s_0$} \Comment{for RAMCP-F only}}
        \State $\badv \gets $ solution to linear program (\ref{eqn:LP})

    \EndWhile
    \State \Return $\pi_\mathrm{avg} =$ \textsc{AvgAction}($h$) for all $h$
  \EndFunction
  \algstore{savept}
  \end{algorithmic}
\end{algorithm}

\section{\algName{} Outline}

Carrying out the GWFP process to solve (\ref{eqn:opt_objective}) requires estimates $\hat{V}_{\pi_k}(i)$ and $\hat{Q}_{b_k}(h,a)$ at every iteration $k$.
To compute $\hat{V}_{\pi_k}(i)$, we can average the total reward accrued on multiple rollouts of policy $\pi_k$ on model $\theta_i$, obtaining a Monte Carlo estimate of $\E \left[ J(\tau) \mid \theta_i, \pi_k \right]$.
Computing $\hat{Q}_{b_k}(h,a)$ is equivalent to approximately solving the BAMDP induced by distribution $b_k$. Many sampling-based methods exist to estimate the optimal Q function in a BAMDP.

Guez et al. \cite{guez2013scalable} showed that performing Monte-Carlo tree search where the dynamics parameter $\theta$ is drawn from $b_k$ at the root of the tree at each iteration can accurately estimate the optimal value function $Q^*_{b_k}(h,a)$. These techniques suggest a na\"ive approach to solving for the optimal policy: at each iteration of the GWFP process, one could compute $\BReps(b_k)$ by running a tree search algorithm on $b_k$, and compute $\BReps(\pi_k)$ by rolling out policy $\pi_k$ on each model $\theta_i$ to get $\hat{V}_{\pi_k}(i)$, and then solving the linear program (Equation \ref{eqn:LP}). This is clearly impractical, as each iteration requires solving a new BAMDP. Furthermore, in order for the GWFP process to converge, the suboptimality of the best responses must go to zero as $k\rightarrow\infty$. Thus, each iteration of this na\"ive implementation would require a growing number of samples, increasing the computational challenges with this approach. Critically, we are able to leverage the structure of the GWFP process to obtain an algorithm that converges to the same result, iterating between performing FP iterations and building the tree. This approach requires growing only one tree, which results in substantial efficiency improvement.

\begin{algorithm}[t]
\caption{\label{alg:ramcp-simulate} Simulate}
\begin{algorithmic}[1]
\small
    \algrestore{savept}
    \Function{Simulate}{$h$, $\theta$, $w$}
        \State $N(h) \gets N(h) + 1$
        \State $W(h) \gets W(h) + w$
                    \label{alg:Wcomputation}
        \State $\Vbr \gets 0$

        \If{$\textsc{Len}(h) >= H$ \textbf{or} $h$ is terminal}
            \State \Return $\Vbr$
        \EndIf

        \ForAll{$a\in\mathcal{A}$}
            \State $N(h,a) \gets N(h,a) + 1$
            \State $W(h,a) \gets W(h,a) + w$

            \State $s' \sim T_\theta(s, a)$
            \State $r \gets R(s,a,s')$
            \State $\Vbr' \gets$ \Call{Simulate}{$has'$,$\theta$,$w$}

            \State $\Qha \gets r + \Vbr'$
            \State {\color{red} $Q(h,a) \gets Q(h,a) + \frac{1}{N(h,a)} ( w \Qha - Q(h,a) )$}
                        \label{alg:q_function_update}

            \If{$a == \argmax_{a'} Q(h,a')$}
                \State $\Vbr \gets \Qha$
                \State $W_{br}(h,a) \gets W_{br}(h,a) + w$
                \State {\color{red} $V(h) \gets V(h) + \frac{1}{N(h)} ( w \Vbr - V(h) )$}
                                 \label{alg:value_function_update}
            \EndIf
        \EndFor
        \State \Return $\Vbr$
    \EndFunction
    \algstore{savept}
    \end{algorithmic}
\end{algorithm}

\subsection{Algorithm Overview}
In this section, we present the \algName{} algorithm, which combines simulation-based search with fictitious play iterations to optimize the risk-sensitive, Bayes-adaptive objective (\ref{eqn:opt_objective}). We present two versions of this procedure: \algNameFull{}, a slower, but provably asymptotically optimal algorithm, and the more efficient \algNameInc{}, for which we do not provide a proof of convergence, but observe good performance empirically. The algorithms share the same overall structure, which is detailed in Algorithm \ref{alg:ramcp}. Portions specific to \algNameFull{} are highlighted in blue, while those specific to \algNameInc{} are in red.


To compute a risk-sensitive plan for belief $b_\mathrm{prior}$ from state $s_0$, the agent calls the \textsc{Search} function. The function iterates between simulating rollouts on different transition models, sampled from an adversarial distribution, and using the improved value estimates from these simulations to improve the agent policy and the adversarial distribution.

Employing ideas from conditional Monte Carlo, the algorithm loops over each model $\theta_i$, and computes a weighting $w$ proportional to the current adversarial belief $\badv$. This weighting is applied to the statistics recorded in the tree, allowing the algorithm to estimate quantities as if the models were sampled from $\badv$ rather than looped over deterministically.

The algorithm maintains a tree where each node is state or action along a trajectory from $s_0$. We denote nodes by the history leading to the node. If the node is a state we refer to this history $h$, and those ending in an action, which we denote as $ha$. For every node, we store visitation counts $N(h), N(h,a)$, cumulative visitation weights $W(h), W(h,a)$, and value estimates $V(h), Q(h,a)$ for every node in the tree.

The call to $\textsc{Simulate}(h,\theta,w)$ (Algorithm \ref{alg:ramcp-simulate}) simulates all possible $H$-length action sequences starting at $h$ under the a given model $\theta$, and increments the visitation weights along the resulting trajectories by $w$. The function returns $\Vbr$, the reward sequence obtained when taking actions were consistent with the greedy policy with respect to the current $Q$ values stored in the tree. By keeping a running average of $\Vbr$ from $s_0$ for every model $\theta_i$ in $\Vmodel(i)$, $\Vmodel(i)$ estimates the performance of the agent's average strategy ($\pi_k$ in GWFP) on model $\theta_i$.

\begin{algorithm}
\small
\caption{\label{alg:ramcp-computeQ} ComputeQValues}
  \begin{algorithmic}[1]
    \algrestore{savept}
    \Function{ComputeQValues}{$h$}
        \If{$\textsc{Len}(h) >= H$ or $h$ is terminal}
            \State $V(h) = 0$
            \State \Return $0$
        \EndIf
        \ForAll{$a \in \mathcal{A}$}
            \State $w, r, V' \gets [\,], [\,], [\,]$
            \ForAll{$s' \in \textsc{Children}(ha)$}
                \State \textsc{Append}($w$, $W(has')$)
                \State \textsc{Append}($r$, $R(s,a,s')$)
                \State \textsc{Append}($V'$, \Call{ComputeQValues}{$has'$})
            \EndFor
            \State $Q(h,a) \gets \frac{1}{W(h,a)} \textsc{Sum}(w \odot (r + V') )$
        \EndFor
        \State $V(h) \gets \max_{a} Q(h,a)$
        \State \Return $V(h)$ \label{alg:value_function_est}
    \EndFunction

\end{algorithmic}
\end{algorithm}

In \algNameFull{}, the value estimates in the tree are updated by calling $\textsc{ComputeQValues}$ (Algorithm \ref{alg:ramcp-computeQ}). This function recurses through the entire simulated tree, and computes value estimates starting at the leaf nodes, via dynamic programming according to the empirical transition probabilities
\begin{align}
    Q(h,a) &= \sum_{s'} \hat{p} (s' | h, a) ( R(s,a,s') + V(has') ) \\
    V(h) &= \max_a Q(h,a)
\end{align}
Here, we only sum over visited states, and use $\hat{p} (s' | h, a) = W(has')/W(ha)$ as the empirical transition probability. In the appendix, it is shown that this probability converges to the transition probability corresponding to a BAMDP with $\theta \sim \bar{b}_\textrm{adv}$, where $\bar{b}_\textrm{adv}$
is the average adversary strategy over the iterations of the algorithm. Since this average strategy is exactly $b_k$ in the GWFP algorithm, this step corresponds to computing $\BReps(b_k)$.

Finally, the procedure computes $\badv = \BReps(\pi_k)$ by solving the LP (\ref{eqn:LP}) given the current estimates of agent performance in $\Vmodel$.

Critically, in contrast to the na\"ive approach, \algName{} avoids performing a separate simulation-based search under a new model distribution $b_k$ at every iteration of of the GWFP. Instead, it performs simulations according to $\badv = \BReps(\pi_k)$ combines the averaging of GWFP with the averaging used in the Monte Carlo estimation of the value estimates. Note that the proposed approach does not compute the running averages $\pi_k$ and $b_k$ directly. However, since the mixed strategy $\pi_k$ is what converges to the Nash equilibrium in GWFP, we keep counts $W_\textrm{br}(ha)$, corresponding to whenever action $a$ matches $\BReps(b_k)$. The mixed strategy corresponds to sampling actions at history $h$ with probability proportional to $W_\textrm{br}(ha)$. Procedure \textsc{AvgAction}($h$) samples actions in this fashion. Thus, \algNameFull{} implicitly carries out the GWFP process, and therefore converges to a Nash equilibrium. More formally, we have the following result:

\begin{theorem}[Convergence of \algNameFull{}]
\label{thm:conv}
Let $\pi_k$ denote the output of \algNameFull{} (Algorithm 1) after $k$ iterations of the outer loop. Then, $\lim_{k \to \infty} \pi_k \in \Pi^*$ in probability.
\end{theorem}

The proof of this result is given in the appendix.

Separating Q value estimation from updating visitation counts in the tree through simulation is more convenient to analyze, but requires iterating through the entire tree at each step of the algorithm. In contrast, many tree search algorithms, such as UCT \cite{kocsis2006bandit}, use stochastic approximation to estimate the Q values via iterative updating, rather than recomputation. \algNameInc{} is an incremental version of \algNameFull{}, which does not call \textsc{ComputeQValues}, and instead adds the incremental updates to $Q(h,a)$ and $V(h)$ into \textsc{Simulate} (see Algorithm \ref{alg:ramcp-simulate}, lines \ref{alg:q_function_update} and \ref{alg:value_function_update}).

While tools from stochastic approximation can be used to prove convergence when the underlying distribution is stationary, this is not the case in the incremental setting, as the adversarial distribution over models is being recomputed at every iteration of GWFP. We observe that in practice, this incremental version also performs well. Popular techniques for Monte Carlo tree search based on non-uniform sampling such as UCT interact with the constantly updated adversarial distribution in a way that is hard to theoretically characterize. However, the added step of solving the LP to compute the adversarial belief is only a minor increase in computational effort for most reasonably-sized problems. Thus, analysis on the interaction of these tree search methods on the non-stationary tree is a promising but mathematically complex avenue of future work, that may result in algorithms that further increase robustness at little computational cost.

\section{Experiments}
We present experimental results for an $n$-pull multi-armed bandit problem, as well as for a patient treatment scenario. The bandit problem is designed to show the fundamental features of the \algName{} algorithm, while the patient treatment example is a larger scale example motivated by a real-world challenge. In both experiments, we use the $\textrm{CVaR}_\alpha$ risk metric, which at a given $\alpha$-quantile corresponds to the expectation over the $\alpha$ fraction worst-case outcomes. For $\alpha=1$, this corresponds to the risk-neutral expectation, and in the limit as $\alpha \to 0^+$, this corresponds to the worst-case metric. For CVaR, the risk polytope $\mathcal{B}$ may be stated in closed form \cite{majumdar2017should}.

\begin{figure*}[t!]
    \centering
        \includegraphics[width=\textwidth]{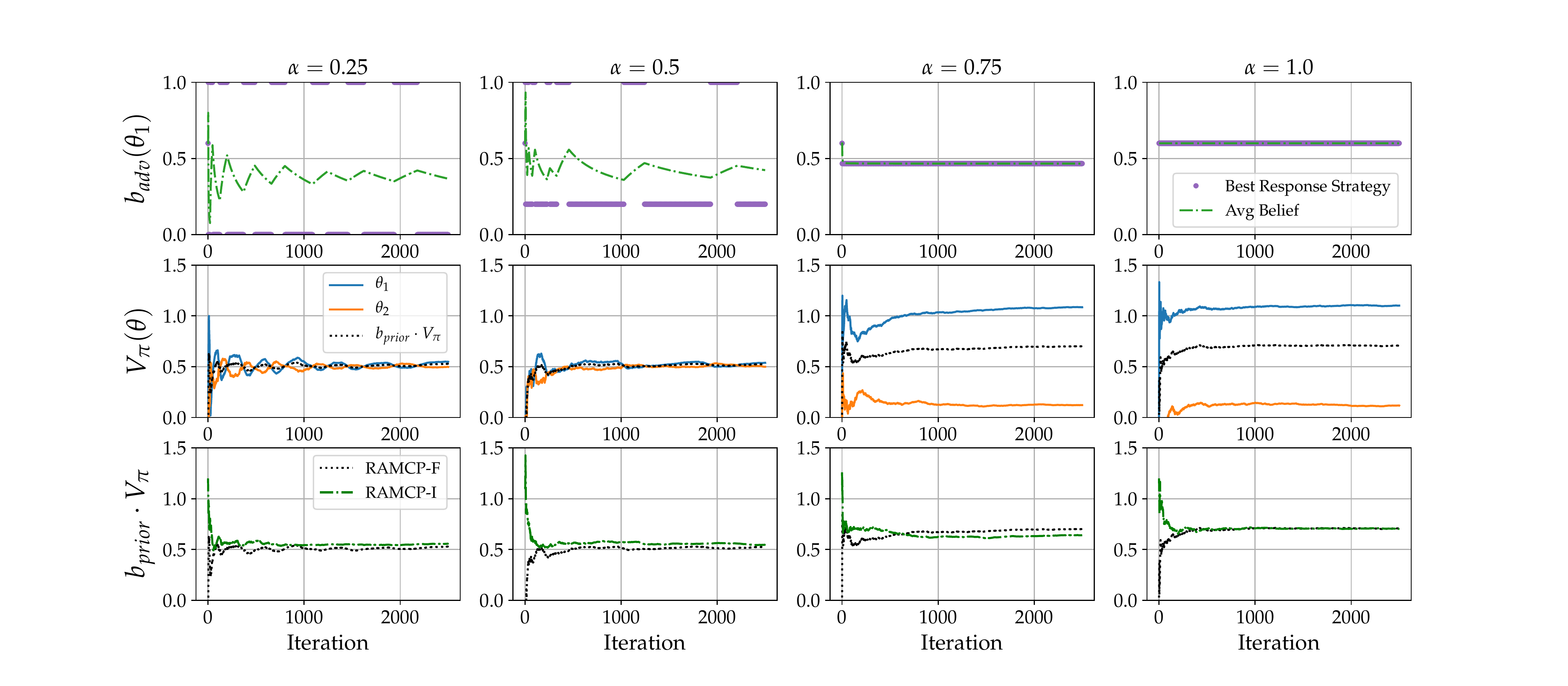}
        \caption{Convergence of the \algNameFull{} algorithm for different CVaR quantiles ($\alpha$). The upper row shows best response belief perturbations (blue) and the running average adversarial belief perturbation (orange) for various CVaR quantiles, $\alpha$. The second row shows the value function for the policy generated by \algNameFull{} operating in an environment with dynamics parameterized by $\theta_1$, and by $\theta_2$ (blue and orange respectively). The dotted line denotes the expected value of the policy under the prior belief over models. The bottom row compares the estimated value under the prior belief as computed by \algNameFull{} against that computed by \algNameInc{}. We see that the two algorithms converge to similar values.}
    \label{fig:bandit_conv}
\end{figure*}

\label{sec:experiments}

\subsection{Multi-armed Bandit}

We consider a multi-armed bandit scenario to illustrate several properties of the \algName{} algorithm. Given a finite nnumber of "pulls" of the bandit, and a finite number of reward realizations, we can represent this scenario as an BAMDP. There is one initial state, in which the agent has four possible actions, which each have different stochastic transition models to 6 possible rewarding states. The rewarding states transition deterministically back to the decision state, where the agent must choose another arm to pull.

We define two possible models, each with different transition probabilities that are known in the planning problem. For illustrative purposes, we chose the transition probabilities to highlight the trade-off between exploitation, exploration and risk. Actions 1 and 2 are \textit{exploratory}: deterministic under each model, and therefore reveal the true model. The actions differ in terms of \textit{risk}: action 1 gives low but similar rewards in both models, while action 2's reward is higher in expectation but varies drastically between the two models. Actions 3 and 4 have more stochastic outcomes under both models, and thus reveal less about the true model. However, these actions serve as choices for \textit{exploitation}, with action 3 offering the highest reward in expectation in $\theta_1$, and action 4 in $\theta_2$.
The prior belief is $(0.6, 0.4)$ for $\theta_1$ and $\theta_2$, respectively. An episode consists of two pulls (or two actions) in the environment.

On this small scale example, we test both \algNameFull{} and \algNameInc{}, and demonstrate that both converge to empirically similar solutions.
Figure \ref{fig:bandit_conv} shows the performance of \algName{} under varying values of $\alpha$. The top row shows the adversarially perturbed belief over iterations of the outer loop of \algNameFull{}. The purple points show the solution to Equation \ref{eqn:LP}. These may switch rapidly, as the adversary will aim to place as little probability mass on the high value models as possible, to minimize the expected value of a given policy. This behavior can be seen for $\alpha = 0.25$. Indeed, this demonstrates the necessity of the averaging in the GWFP process. While the best response beliefs change rapidly, the running average of the adversarially perturbed belief converges, shown in orange. In the second row of Figure \ref{fig:bandit_conv}, the estimated values $\hat{V}_{\pi_k}$ for each $\theta_i$ are plotted. The dotted line denotes the expected value of these estimates with respect to the prior belief. As the CVaR quantile $\alpha$ decreases, the expected reward with respect to the prior belief over models decreases, but the performance for the worst-case model improves. This illustrates how by choosing the risk metric used by \algNameFull{} (here by tuning the $\alpha$-quantile of the CVaR metric), the user can obtain policies that meet their standards for robustness. The last row of Figure \ref{fig:bandit_conv} compares the value under the prior belief as estimated by \algNameFull{} against that estimated by \algNameInc{}. We see that \algNameInc{} converges to similar values as \algNameFull{}, suggesting that it is a reasonable approximation to the asymptotically optimal algorithm.

Figure \ref{fig:bandit_prior} shows the mean performance of the policies obtained through both \algNameFull{} and \algNameInc{} on the bandit problem. We plot the mean performance of the algorithm as the distribution over underlying models is allowed to shift adversarially away from the prior distribution that was used in planning. We see that \algName{} is able to sacrifice expected performance under the provided prior in exchange for robustness to an incorrect prior distribution. Further, we see that both \algNameFull{} and \algNameInc{} have the same performance, which further supports \algNameInc{} as a practically useful modification of \algNameFull{}.

\begin{figure}[t]
    \includegraphics[width=0.95\linewidth]{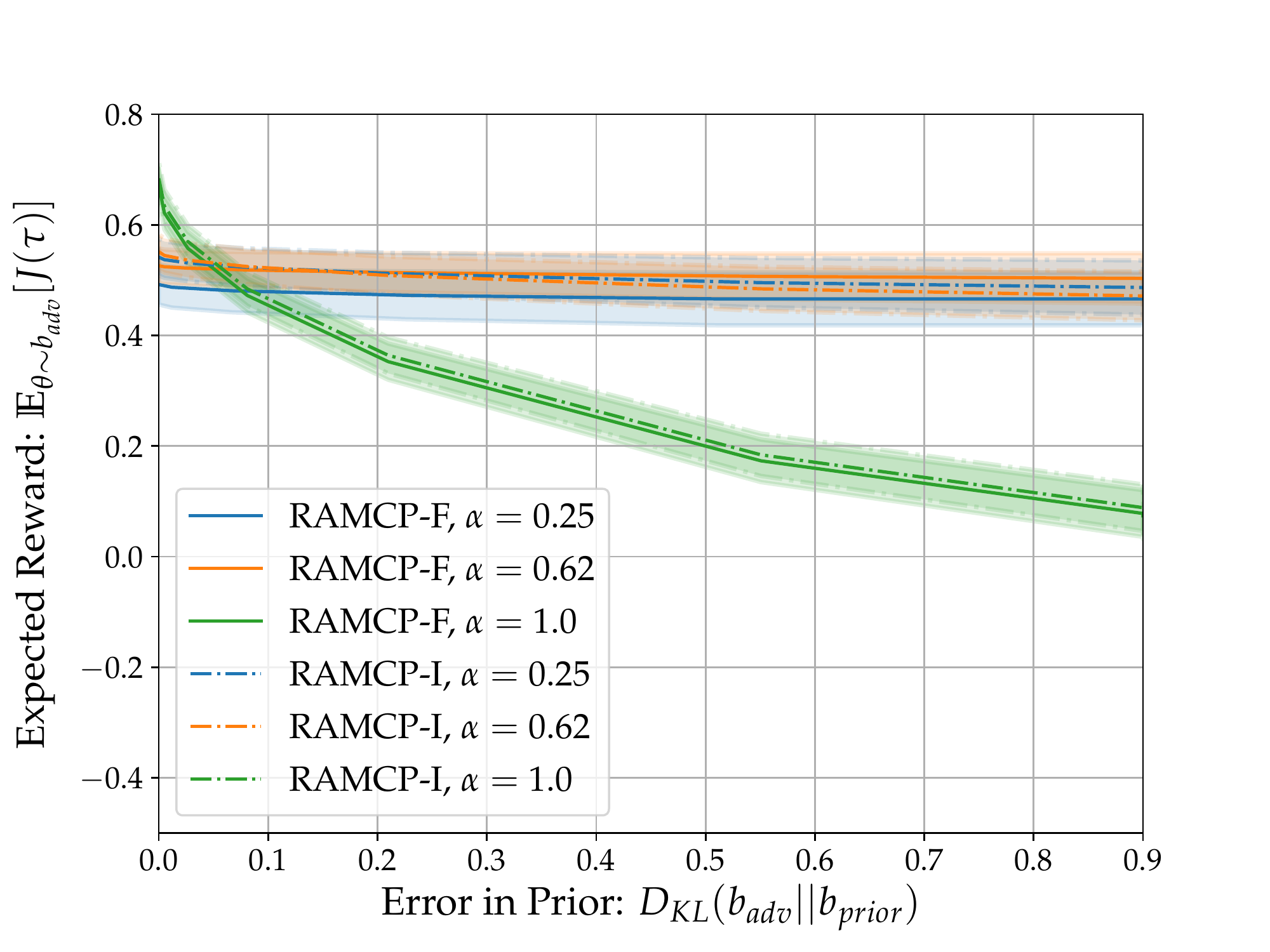}
    \caption{Expected total reward obtained in the bandit problem if the prior is inaccurate. Note that by tuning the CVaR quantile ($\alpha$) used in the RAMCP algorithm, the resulting policy sacrifices expected performance under the prior to become more robust against distributions over $\theta$ that differ from the prior. Error bars denote 90\% confidence bounds on the mean total reward. We observe that on this problem, the \algNameInc{} algorithm matches the performance of the asymptotically optimal \algNameFull{}}
    \label{fig:bandit_prior}
\end{figure}

\subsection{Patient Treatment}

The problem of developing patient treatment plans is one where being robust yet adaptive to model uncertainty is critical. Each individual patient may respond in different ways to a given treatment strategy, and therefore exploratory actions for model disambiguation are often required. There are many such problems in medicine that have been formulated under the BAMDP or POMDP framework, including choosing drug infusion regimens \cite{hu1996comparison}, or HIV treatment plans \cite{attarian2017optimal}. With human lives at stake, being robust towards this model uncertainty is also critical, and thus RAMCP offers an attractive solution approach. We demonstrate the effectiveness of RAMCP in such a domain by developing a simplified problem which captures the complexity of such tasks.

We consider a model where the state is the patient's health: $s \in \mathcal{S} = \{0,1,\dots,19\}$, where $s=0$ corresponds to death, and $s=19$ corresponds to a full health level. The patient starts at $s=3$. The action space is $\mathcal{A} = \{1,2,3\}$, corresponding to different treatment options. We randomly generated 15 possible response profiles to each treatment option. Each response profile assigns a probability mass to a relative change in patient health. Under the different response profiles, this change may be positive or negative, so rapid model identification is important.
The exact method used to generate the transition probabilities is provided in the supplementary materials.
We consider a prior which assigns a weight of $0.25$ to the first response profile, and 0.0536 to the remaining 14.

Figure \ref{fig:patient_prior} shows the performance of the RAMCP algorithm on this problem, run for 12500 iterations with a search horizon $H=4$, 500 times for each $\alpha$-quantile. As can be seen in the figure, the choice of the CVaR quantile $\alpha$ controls the robustness of the resulting plan to an incorrect prior. The risk-neutral setting ($\alpha = 1.0$) yields the highest expected performance assuming an accurate prior, but its performance can degrade rapidly as the distribution that the model is sampled from changes. Rather than force a practitioner into this tradeoff, \algName{} allows them to encode a preference towards robustness by altering the risk metric. The more robust setting of $\alpha = 0.2$ yields a policy with performance that is less sensitive to the underlying distribution over models, at the cost of a worse reward under the prior distribution. Notably, $\alpha = 0.6$ offers the practitioner an attractive middle ground, offering improved robustness compared to the risk-neutral policy at a small cost to performance on the prior.

We compared the policies obtained by optimizing the coherent CVaR objective through \algName{} against policies optimized with an exponentially weighted objective $\tilde{r} = \exp(-\gamma r)$, a common heuristic for incorporating risk-sensitivity into optimization. While reweighting the reward in this way allows using standard BAMDP algorithms and avoids the game-theoretic formulation, the exponential risk formulation is not a coherent risk metric. Furthermore, it does not allow addressing the model uncertainty separately from the transition stochasticity. Experimentally, this is evident in Figure \ref{fig:patient_prior}. While using exponential reward shaping to add robustness to the BAMDP objective does give resulting policies that are somewhat less sensitive to the underlying prior, the expected performance drops under all possible model distributions.

\begin{figure}[t]
    \includegraphics[width=0.95\linewidth]{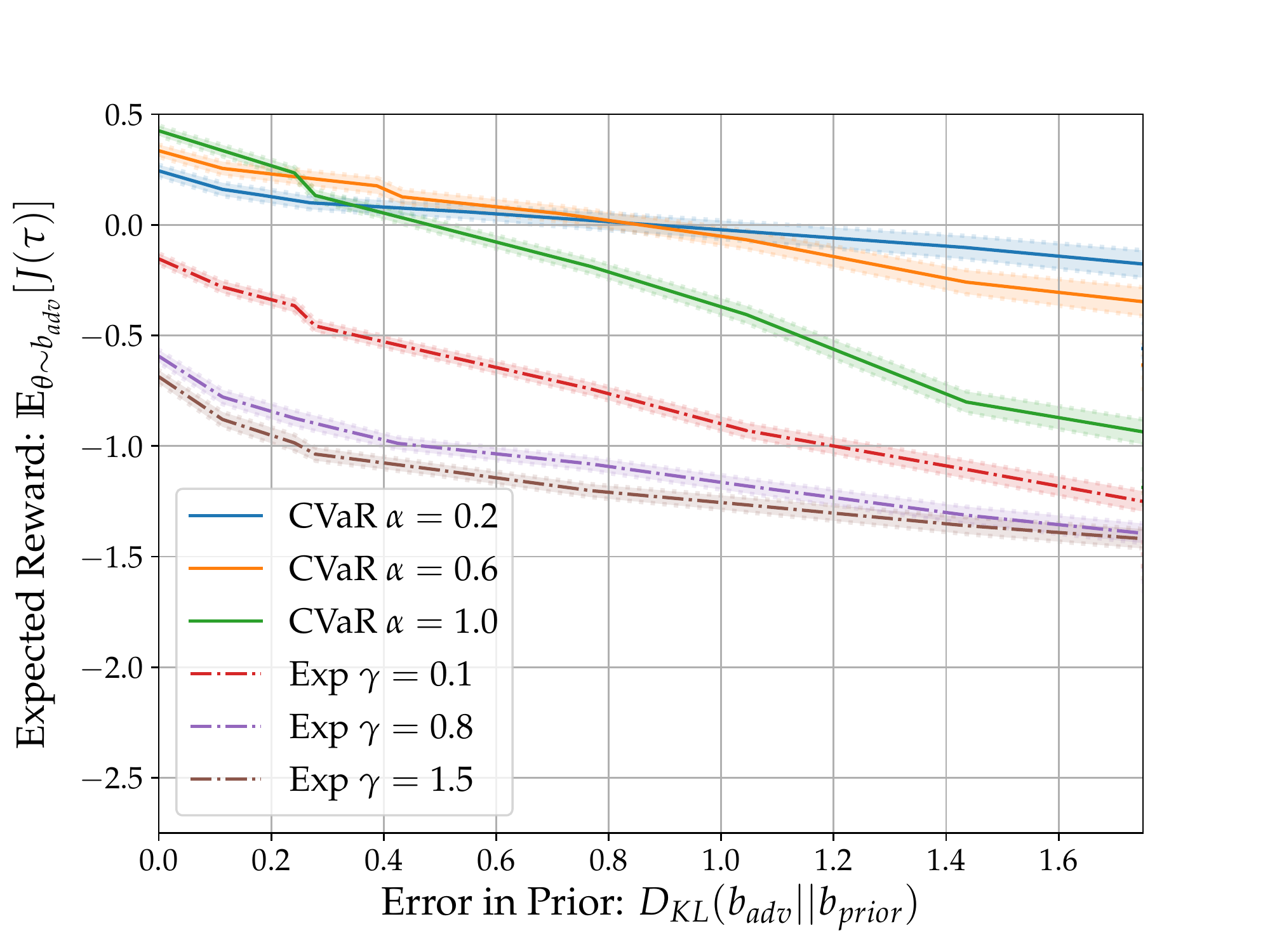}
    \caption{Expected total reward obtained in the patient treatment problem if the prior is inaccurate. As in the bandit problem, tuning the CVaR quantile ($\alpha$) enables the \algNameInc{} algorithm to choose plans that are more robust against distributions over $\theta$ that differ from the prior. A na\"ive approach to add risk sensitivity to the BAMDP via exponential cost shaping yields policies that perform consistently worse, as shown in the dashed-dotted curves. Error bars denote 90\% confidence bounds on the mean total reward.} \vspace{-1mm}
    \label{fig:patient_prior}
\end{figure}

\section{Related Work}
\label{sec:related}

This work aims to optimally trade off between exploration, exploitation, and robustness in the problem of planning under model uncertainty. This problem has been approached in the bandit setting \cite{galichet2013exploration}, but little work has been done in the MDP setting. Generally, \algName{} can be seen as an extension of the Bayes-Adaptive MDP literature toward safety and robustness. In the discrete MDP case, approaches toward the BAMDP problem are often applications of efficient search techniques for MDPs and POMDPs \cite{wang2005bayesian} \cite{guez2013scalable}. Approaches in the continuous case often use heuristics \cite{bai2013planning} or leverage tools from adaptive control theory \cite{slade2017simultaneous}, \cite{yu2017preparing}. However, while tools from adaptive control are capable of guaranteeing safety while optimizing performance, they typically do not consider the value of information in performance optimization \cite{aastrom2013adaptive}.

Robust MDPs \cite{nilim2005robust} are a popular tool for planning under model uncertainty, however they have several features which produces undesirable behavior. First, these problems are typically computationally intractable for general uncertainty sets \cite{bagnell2001solving}. Additionally, this framework often leads to extremely conservative policies. Finally, in the BAMDP setting, information gain that simply reweights the probability of elements in an uncertainty set would not change the robust objective, and so the interaction with Bayes-adaptive models is poor. Several of the weaknesses of the robust MDP framework are mitigated by risk-averse approaches to policy optimization \cite{howard1972risk}, which allow a more naturally tunable notion of conservatism \cite{chow2015risk}. However, both the risk-sensitive and robust MDP frameworks do not generally consider that the uncertainty over dynamics models may be reduced as the agent interacts with the environment, still leading to unnecessarily conservative policies that do not consider the value of information gathering actions. We address this problem by introducing notions of risk-sensitivity to BAMDPs.

\section{Discussion \& Conclusions}
\label{sec:conclusion}

This paper has introduced \algName{}, which marks what we believe to be a first attempt toward efficient methods of computing policies that optimally balance exploration, exploitation, and robustness. This approach has demonstrated good experimental performance on discrete MDPs, typically requiring only minor increases in computation over standard BAMDP solution approaches. There are several clear directions for future work in the effort to bring this tradeoff between safety and performance to more complex domains. First, further analysis into the \algNameInc{} formulation of \algName{} are promising, as the algorithm achieves good performance in practice. Second, through the use of value function approximation, the algorithm could be applied to larger, even continuous MDPs. Such approaches have shown good performance for risk-neutral BAMDPs \cite{guez2014bayes}. Another promising direction for future work is to investigate if the algorithm can be modified to maintain convergence guarantees with more advanced sampling strategies such as UCT \cite{kocsis2006bandit}. Finally, this work focused on planning over a discrete distribution over models. Future work may investigate whether sampling strategies such as sequential Monte Carlo can be used to extend \algName{} to continuous beliefs over models parameters.

\subsubsection*{Acknowledgements}
Apoorva Sharma was supported in part by the Stanford School of Engineering Fellowship. James Harrison was supported in part by the Stanford Graduate Fellowship and the National Sciences and Engineering Research Council (NSERC). The authors were partially supported by the Office of Naval Research, Young Investigator Program, and by DARPA, Assured Autonomy program.

\bibliographystyle{aaai}
\bibliography{biblio}

\begin{thebibliography}{}

\bibitem[\protect\citeauthoryear{Artzner \bgroup et al\mbox.\egroup
  }{1999}]{artzner1999coherent}
Artzner, P.; Delbaen, F.; Eber, J.-M.; and Heath, D.
\newblock 1999.
\newblock Coherent measures of risk.
\newblock {\em Mathematical finance}.

\bibitem[\protect\citeauthoryear{Astr{\"o}m and
  Wittenmark}{2013}]{aastrom2013adaptive}
Astr{\"o}m, K.~J., and Wittenmark, B.
\newblock 2013.
\newblock {\em Adaptive control}.
\newblock Courier Corporation.

\bibitem[\protect\citeauthoryear{Attarian and Tran}{2017}]{attarian2017optimal}
Attarian, A., and Tran, H.
\newblock 2017.
\newblock An optimal control approach to structured treatment interruptions for
  hiv patients: a personalized medicine perspective.
\newblock {\em Applied Mathematics}.

\bibitem[\protect\citeauthoryear{Bagnell, Ng, and
  Schneider}{2001}]{bagnell2001solving}
Bagnell, J.~A.; Ng, A.~Y.; and Schneider, J.~G.
\newblock 2001.
\newblock Solving uncertain markov decision processes.
\newblock {\em Neural Information Processing Systems}.

\bibitem[\protect\citeauthoryear{Bai, Hsu, and Lee}{2013}]{bai2013planning}
Bai, H.; Hsu, D.; and Lee, W.~S.
\newblock 2013.
\newblock Planning how to learn.
\newblock {\em International Conference on Robotics and Automation}.

\bibitem[\protect\citeauthoryear{Brown}{1949}]{brown1949some}
Brown, G.~W.
\newblock 1949.
\newblock Some notes on computation of games solutions.
\newblock Technical report, RAND Corporation.

\bibitem[\protect\citeauthoryear{Chen \bgroup et al\mbox.\egroup
  }{2016}]{chen2016pomdp}
Chen, M.; Frazzoli, E.; Hsu, D.; and Lee, W.~S.
\newblock 2016.
\newblock {POMDP}-lite for robust robot planning under uncertainty.
\newblock {\em International Conference on Robotics and Automation}.

\bibitem[\protect\citeauthoryear{Chow \bgroup et al\mbox.\egroup
  }{2015}]{chow2015risk}
Chow, Y.; Tamar, A.; Mannor, S.; and Pavone, M.
\newblock 2015.
\newblock Risk-sensitive and robust decision-making: a {CVaR} optimization
  approach.
\newblock {\em Neural Information Processing Systems}.

\bibitem[\protect\citeauthoryear{Duff}{2002}]{duff2002optimal}
Duff, M.~O.
\newblock 2002.
\newblock {\em Optimal Learning: Computational procedures for Bayes-adaptive
  Markov decision processes}.
\newblock Ph.D. Dissertation, University of Massachusetts at Amherst.

\bibitem[\protect\citeauthoryear{Eichhorn and
  R{\"o}misch}{2005}]{eichhorn2005polyhedral}
Eichhorn, A., and R{\"o}misch, W.
\newblock 2005.
\newblock Polyhedral risk measures in stochastic programming.
\newblock {\em Journal on Optimization}.

\bibitem[\protect\citeauthoryear{Galichet, Sebag, and
  Teytaud}{2013}]{galichet2013exploration}
Galichet, N.; Sebag, M.; and Teytaud, O.
\newblock 2013.
\newblock Exploration vs exploitation vs safety: Risk-aware multi-armed
  bandits.
\newblock {\em Asian Conference on Machine Learning}.

\bibitem[\protect\citeauthoryear{Ghavamzadeh \bgroup et al\mbox.\egroup
  }{2015}]{ghavamzadeh2015bayesian}
Ghavamzadeh, M.; Mannor, S.; Pineau, J.; Tamar, A.; et~al.
\newblock 2015.
\newblock Bayesian reinforcement learning: A survey.
\newblock {\em Foundations and Trends in Machine Learning}.

\bibitem[\protect\citeauthoryear{Gittins, Glazebrook, and
  Weber}{2011}]{gittins2011multi}
Gittins, J.; Glazebrook, K.; and Weber, R.
\newblock 2011.
\newblock {\em Multi-armed bandit allocation indices}.
\newblock John Wiley \& Sons.

\bibitem[\protect\citeauthoryear{Glover and Doyle}{1988}]{glover1988state}
Glover, K., and Doyle, J.~C.
\newblock 1988.
\newblock State-space formulae for all stabilizing controllers that satisfy an
  h$_\infty$-norm bound and relations to relations to risk sensitivity.
\newblock {\em Systems \& control letters}.

\bibitem[\protect\citeauthoryear{Guez \bgroup et al\mbox.\egroup
  }{2014}]{guez2014bayes}
Guez, A.; Heess, N.; Silver, D.; and Dayan, P.
\newblock 2014.
\newblock Bayes-adaptive simulation-based search with value function
  approximation.
\newblock {\em Neural Information Processing Systems}.

\bibitem[\protect\citeauthoryear{Guez, Silver, and
  Dayan}{2013}]{guez2013scalable}
Guez, A.; Silver, D.; and Dayan, P.
\newblock 2013.
\newblock Scalable and efficient bayes-adaptive reinforcement learning based on
  monte-carlo tree search.
\newblock {\em Journal of Artificial Intelligence Research}.

\bibitem[\protect\citeauthoryear{Heinrich, Lanctot, and
  Silver}{2015}]{heinrich2015fictitious}
Heinrich, J.; Lanctot, M.; and Silver, D.
\newblock 2015.
\newblock Fictitious self-play in extensive-form games.
\newblock {\em International Conference on Machine Learning}.

\bibitem[\protect\citeauthoryear{Howard and Matheson}{1972}]{howard1972risk}
Howard, R.~A., and Matheson, J.~E.
\newblock 1972.
\newblock Risk-sensitive markov decision processes.
\newblock {\em Management science}.

\bibitem[\protect\citeauthoryear{Hu, Lovejoy, and
  Shafer}{1996}]{hu1996comparison}
Hu, C.; Lovejoy, W.~S.; and Shafer, S.~L.
\newblock 1996.
\newblock Comparison of some suboptimal control policies in medical drug
  therapy.
\newblock {\em Operations Research}.

\bibitem[\protect\citeauthoryear{Kocsis and
  Szepesv{\'a}ri}{2006}]{kocsis2006bandit}
Kocsis, L., and Szepesv{\'a}ri, C.
\newblock 2006.
\newblock Bandit based monte-carlo planning.
\newblock {\em European Conference on Machine Learning}.

\bibitem[\protect\citeauthoryear{Leslie and
  Collins}{2006}]{leslie2006generalised}
Leslie, D.~S., and Collins, E.~J.
\newblock 2006.
\newblock Generalised weakened fictitious play.
\newblock {\em Games and Economic Behavior}.

\bibitem[\protect\citeauthoryear{Majumdar and
  Pavone}{2017}]{majumdar2017should}
Majumdar, A., and Pavone, M.
\newblock 2017.
\newblock How should a robot assess risk? {Towards} an axiomatic theory of risk
  in robotics.
\newblock {\em International Symposium on Robotics Research}.

\bibitem[\protect\citeauthoryear{Martin}{1967}]{martin1967bayesian}
Martin, J.~J.
\newblock 1967.
\newblock {\em Bayesian decision problems and Markov chains}.
\newblock Wiley.

\bibitem[\protect\citeauthoryear{Nilim and El~Ghaoui}{2005}]{nilim2005robust}
Nilim, A., and El~Ghaoui, L.
\newblock 2005.
\newblock Robust control of markov decision processes with uncertain transition
  matrices.
\newblock {\em Operations Research}.

\bibitem[\protect\citeauthoryear{Osogami}{2012}]{osogami2012robustness}
Osogami, T.
\newblock 2012.
\newblock Robustness and risk-sensitivity in markov decision processes.
\newblock {\em Neural Information Processing Systems}.

\bibitem[\protect\citeauthoryear{Rabin and Thaler}{2001}]{rabin2001anomalies}
Rabin, M., and Thaler, R.~H.
\newblock 2001.
\newblock Anomalies: risk aversion.
\newblock {\em Journal of Economic perspectives}.

\bibitem[\protect\citeauthoryear{Slade \bgroup et al\mbox.\egroup
  }{2017}]{slade2017simultaneous}
Slade, P.; Culbertson, P.; Sunberg, Z.; and Kochenderfer, M.~J.
\newblock 2017.
\newblock Simultaneous active parameter estimation and control using
  sampling-based {B}ayesian reinforcement learning.
\newblock {\em International Conference on Intelligent Robots and Systems}.

\bibitem[\protect\citeauthoryear{Tamar \bgroup et al\mbox.\egroup
  }{2017}]{tamar2017sequential}
Tamar, A.; Chow, Y.; Ghavamzadeh, M.; and Mannor, S.
\newblock 2017.
\newblock Sequential decision making with coherent risk.
\newblock {\em Transactions on Automatic Control}.

\bibitem[\protect\citeauthoryear{Wang \bgroup et al\mbox.\egroup
  }{2005}]{wang2005bayesian}
Wang, T.; Lizotte, D.; Bowling, M.; and Schuurmans, D.
\newblock 2005.
\newblock Bayesian sparse sampling for on-line reward optimization.
\newblock {\em International Conference on Machine Learning}.

\bibitem[\protect\citeauthoryear{Yu, Liu, and Turk}{2017}]{yu2017preparing}
Yu, W.; Liu, C.~K.; and Turk, G.
\newblock 2017.
\newblock Preparing for the unknown: Learning a universal policy with online
  system identification.
\newblock {\em Robotics: Science and Systems}.

\end{thebibliography}

\clearpage
\appendix

\section{Proof of Theorem \ref{thm:conv}}
\label{sec:gwfp}

\subsection{Generalized Weakened Fictitious Play}

In this section we prove Theorem \ref{thm:conv}, and provide necessary background material. We begin by formally introducing Generalized Weakened Fictitious Play (GWFP) \cite{leslie2006generalised}, and restating the core convergence results from that work.

Generally, we consider a repeated $N$-player normal-form game. We will write the pure strategy set of player $i$ as $A^i$, and the mixed strategy set as $\Delta^i$, where a mixed strategy is a distribution over pure strategies. Let $r^i : \times_{i=1}^{N} \Delta^i \to \mathbb{R}$ denote the bounded reward function of player $i$ (where $\times_{i=1}^{N}$ denotes the Cartesian product). Then, letting the $\pi^{-i}$ denote a set of mixed strategies for all players but player $i$, let $r^i(\pi^i, \pi^{i-1})$ denote the expected reward to player $i$ selecting strategy $\pi^i$, if all other players select $\pi^{i-1}$. We will write the best response of player $i$ to $\pi^{i-1}$ as
\begin{equation*}
    \BR^i(\pi^{i-1}) = \arg \max_{\pi^i \in \Delta^i} r^i(\pi^i, \pi^{i-1}),
\end{equation*}
and thus, we can write the collection of best responses as
\begin{equation*}
    \BR(\pi) = \times_{i=1}^{N} \BR^i(\pi^{i-1}).
\end{equation*}
Finally, in a similar fashion, we will define $\epsilon$-best response of player $i$ to be the set
\begin{equation*}
    \BReps^i(\pi^{-i}) = \left\{\pi^i \in \Delta^i : r^i(\pi^i, \pi^{-i}) \geq r^i(\BR^i(\pi^{-i}), \pi^{-i})-\epsilon \right\}.
\end{equation*}
The set of $\epsilon$-best strategies is defined in the same way as above, as is written $BR_\epsilon(\pi)$. We may now formally define the GWFP update process.

\begin{definition}[Generalized Weakened Fictitious Play (GWFP) Process \cite{leslie2006generalised}] A GWFP process is any process $\{\sigma_n\}_{m\geq0}$, with $\sigma_n \in \times^{N}_{i=1} \Delta^i$, such that
\begin{equation}
\label{eqn:gwfp}
    \sigma_{n+1} \in (1-\alpha_{n+1}) \sigma_n + \alpha_{n+1} (\BR_{\epsilon_n}(\sigma_n) + M_{n+1})
\end{equation}
with $\alpha_n \to 0$ and $\epsilon_n \to 0$ as $n \to \infty$,
\begin{equation*}
    \sum_{n\geq1}\alpha_n = \infty,
\end{equation*}
and $\{M_n\}_{n\geq1}$ is a sequence of perturbations such that, for any $T>0$,
\begin{equation*}
    \lim_{n \to \infty} \sup_{k} \left\{ \left\| \sum_{i=n}^{k-1} \alpha_{i+1} M_{i+1}\right\| : \sum_{i=n}^{k-1} \alpha_{i+1} \leq T \right\} = 0.
\end{equation*}
\end{definition}

Equation \ref{eqn:gwfp} is written as an inclusion, and thus can be thought of as defining a set of possible next strategy profiles. In practice, we only compute one strategy profile which lies in the set $\BReps(\pi)$, and assign $\sigma_{n+1}$ to the corresponding mixed profile. Thus, the reader should think of this equation as defining an iterative updating of the strategy profiles.

Critically, this definition allows us to establish guarantees on convergence in certain cases.

\begin{lemma}[Convergence of GWFP Processes]
\label{lem:conv}
Any GWFP process will converge to the set of Nash equilibria in two-player zero-sum games, in potential games, and in generic $2 \times m$ games.
\end{lemma}
\begin{proof}
Restatement of Corollary 5 in \cite{leslie2006generalised}.
\end{proof}

\subsection{Model Value Estimation}

Having formalized the GWFP process and defined relevant notation, we next define notation specific to the \algName{} algorithm, and relate the terms to lines in the algorithm and variables in the GWFP process. In \algName{} the two agents are the adversary, whose strategy $\badv$ is a mixed strategy among different model choices, and the agent, whose strategy is a history dependent policy $\pi$. The approximate best response calculations for both agents are carried out through maintaining value estimates and then acting optimally assuming the value estimates are accurate. Therefore, the suboptimality of the best-response policies $\BReps(\pi_k)$ and $\BReps(b_k)$ is bounded by the errors of the value estimates $\hat{V}_{\pi_k}$ and $\hat{Q}_(h,a)$ respectively. We therefore begin by showing that as $k \rightarrow \infty$, these errors go to zero in probability.

For the purposes of the following proof, we assume without loss of generality that the total rewards over the planning horizon are bounded with a range of $V_\textrm{range}$. Furthermore, throughout this section we assume the $b_k$ and $\pi_k$ are the belief and agent policy that follow the iterative averaging process that is analogous to the GWFP process defined above. For ease of notation, we refer to the approximate best responses played at iteration $k$ by the adversary as $b^*_k \in BR_\epsilon(\pi_{k-1})$, and those played by the agent as $\pi^*_k \in BR_\epsilon(\pi_{k-1})$.

\begin{lemma}[Convergence of Model Value Estimates]
\label{lem:v}
Let $V_{\pi_k}(\theta_i)$ correspond to the expected reward of policy $\pi_k$ on model $\theta_i$. Let $\hat{V}_{k}(\theta_i)$ be the estimate calculated by the iterative update in line \ref{alg:model_value_update} of Algorithm \ref{alg:ramcp}. Defining the error at iteration $k$ as $\epsilon_k = |\hat{V}_{k}(\theta_i) - V_{\pi_k}(\theta_i)|$, as $k \rightarrow \infty$, $\epsilon_k \rightarrow 0$ in probability.
\end{lemma}

\begin{proof}
To show that the iterative estimates $\hat{V}_k(\theta_i)$ serve as an unbiased estimator for $V_{\pi_k}(\theta_i)$, we leverage the fact that $\pi_k$ is a moving average
\begin{align*}
    \pi_k &= (1-\frac{1}{k})\pi_{k-1} + \frac{1}{k}\pi^*_k \\
          &= \frac{1}{k} \sum_{i=1}^{k} \pi^*_i.
\end{align*}
Note that this average over policies represents to a mixing of strategies, i.e. $\pi_k$ is the mixed strategy that randomly picks between the $k$ past best-response policies, and plays that policy. This means that the value of this mixed strategy will be itself a moving average:
\begin{equation*}
    V_{\pi_k}(\theta_i) = (1-\frac{1}{k})V_{\pi_{k-1}}(\theta_i) + \frac{1}{k} V_{\pi^*_k}(\theta_i)
\end{equation*}

In \algName, the approximate best-response policy corresponds to the policy that is greedy with respect to the Q estimates at the start of iteration $k$:  $\pi^*_k(h) = \arg\max_a \hat{Q}_{k-1}(h,a)$. Thus, the total reward obtained when following this greedy policy under model $\theta_i$ is a random variable $V_{\textrm{br},k}$ whose expectation is $V_{\pi^*_k}$. We see that the iterative update in line \ref{alg:model_value_update} of of Algorithm \ref{alg:ramcp} computes $\hat{V}_{\pi_k}(\theta_i)$ as a running average of these $V_{\textrm{br},k}$, and therefore $\E[ \hat{V}_{k}(\theta_i) ] = V_{\pi_k}(\theta_i)$.

As $V_{\textrm{br},k}$ is bounded with a range of $V_\textrm{range}$, Hoeffding's inequality gives
\begin{equation}\label{eqn:model_val_bound}
    P(|\hat{V}_{k}(\theta_i) - V_{\pi_k}(\theta_i)| \ge \epsilon) \le 2 \exp ( - \frac{2k^2\epsilon^2}{V_\textrm{range}^2} )
\end{equation}
which implies that as $k \rightarrow \infty$, $\epsilon_k \rightarrow 0$ in probability.

\end{proof}

\subsection{Tree Value Estimation}

Convergence of the estimates of the value by the tree depends on a small modification of the Strong Law of Large Numbers (SLLN), presented for completeness:

\begin{theorem}[SLLN for independent but not identically distributed random variables]
Let $\{X_k\}_{k=1}^n$ be independent random variables so that $\text{Var}(X_k) < \infty$ for each $k \in \mathbb{N}$ and $\sum_{k=1} \frac{1}{k^2} \text{Var}(X_k) < \infty$. Then
\begin{align}\label{thm:SLLN}
\mathbb{P} \left( \lim_{n \rightarrow \infty} \frac{1}{n} \sum_{k=1}^n \left( X_k - \mathbb{E}[X_k] \right) = 0 \right) = 1.
\end{align}
\end{theorem}

Equipped with this, we may formally state the convergence for the tree value estimates of \algNameFull{}. Let $V_k$ denote the tree value estimate, as computed in line \ref{alg:value_function_est} of Algorithm \ref{alg:ramcp-computeQ} at iteration $k$ of the algorithm.

In each iteration of the tree search, we sample a trajectory $h$ of horizon $H$ from the generative process $\theta \sim p(\theta), h \sim p(h\mid\theta)$. Note that $p(h\mid\theta)$ also depends on the action selection policy, but since in this work we sample all possible action sequences at every iteration, we do not explicitly write this dependence. In the algorithm, we sample trajectories by sampling the model where $p(\theta) = q(\theta)$, where $q(\theta)$ is the uniform distribution over the $M$ models. At iteration $k$ of GWFP, the agent must provide a near-optimal strategy for the generative process where $p(\theta) = b_k(\theta) = \frac{1}{k} \sum_{j=1}^k \badvj(\theta)$. Let $\Vstark$ denote the value function under this generative process.

Note that each of these processes define a marginal distribution over trajectories, which we denote as $q(h) = \sum_{i=1}^M q(\theta_i) p(h \mid \theta_i)$ and $\bar{p}_k(h) = \sum_{i=1}^M b_k(\theta_i) p(h \mid \theta_i)$.

In order to show that the estimates $V_k$ computed by \algNameFull{} converge to $\Vstark$, we also make use of the following assumption:

\begin{assumption}\label{as:bounded_v}
We assume that the value functions $V_k, \Vstark$ are all bounded by some positive constant $V_{\max} < \infty$ with probability 1.
\end{assumption}

Due to boundedness of rewards, the second assumption is a result of either operating within finite horizon problems, using a discount factor, or having a nonzero probability of transitioning to an absorbing terminal state at each state.

\begin{lemma}[Convergence of Tree Value Estimates]
\label{lem:q}
As $k\rightarrow\infty$, the value estimate $V_k$ under \algNameFull{} converges to $\Vstark$ almost surely.
\end{lemma}

\begin{proof}
To begin, we analyze the convergence of the statistics $W(h)$ estimated by the tree. From line \ref{alg:Wcomputation} in Algorithm \ref{alg:ramcp-simulate}, we can see that after sampling $k$ trajectories in the tree search, $W(h)$ is simply the sum of the weight $w_j$ over all simulations that passed through trajectory $h$.
\begin{align*}
   W(h) = \sum_{j=1}^k w_j \mathds{1}\{h_j = h\}
\end{align*}
Taking the expectation of this quantity over $q$, the generative process used in \algNameFull{}, and plugging the expression used to compute the weight $w_j$, we obtain:
\begin{align*}
    \E_q \left[ \frac{1}{k} W(h) \right] &= \sum_{j=1}^k \E_q \left[ w_j \mathds{1}_{\{h_j = h\}} \right] \\
    &= \frac{1}{k} \sum_{j=1}^k \E_{\theta \sim q(\theta)} \left[ \E \left[ w_j \mathds{1}_{\{h_j = h\}} \mid \theta \right] \right] \\
    &= \frac{1}{k} \sum_{j=1}^k \sum_{i=1}^M q(\theta_i) \E \left[ w_j \mathds{1}_{\{h_j = h\}} \mid \theta = \theta_i \right]  \\
    &= \frac{1}{k} \sum_{j=1}^k \sum_{i=1}^M q(\theta_i) \E \left[ \frac{\badvj(\theta_j)}{q(\theta_j)} \mathds{1}_{\{h_j = h\}} \mid \theta_j = \theta_i \right]  \\
    &= \frac{1}{k} \sum_{j=1}^k \sum_{i=1}^M  \badvj(\theta_i) p(h \mid \theta_i) \\
    &= \sum_{i=1}^M p(h \mid \theta_i) \left( \frac{1}{k} \sum_{j=1}^k  \badvj(\theta_i) \right) \\
    &= \sum_{i=1}^M p(h \mid \theta_i) b_k(\theta_i)\\
    &= \bar{p}_k(h).
\end{align*}
The simulation at each iteration is independent, so invoking the SLLN, we see that
\begin{align*}
    &\mathbb{P} \left( \lim_{k \rightarrow\infty} \frac{1}{k} \sum_{j=1}^k \left( w_j \mathds{1}_{\{h_j = h\}} - \E_q[w_j \mathds{1}_{\{h_j = h\}}] \right) = 0 \right) = 1 \\
&\iff \mathbb{P} \left( \lim_{k \rightarrow\infty} \frac{1}{k} W(h) - \bar{p}(h) = 0 \right) = 1.
\end{align*}

When computing the value function estimates, we take expectations w.r.t. the empirical transition probability computed as:
$$\hat{p}_k(s'|h,a) = \frac{ W(has') }{ W(ha) } = \frac{\frac{1}{k} W(has')}{\frac{1}{k} W(ha)}.$$
By the result above, we know that the numerator converges almost surely to $\bar{p}_k(has')$, and denominator to $\bar{p}_k(ha)$, and thus by continuity,
\begin{align}
    \mathbb{P} \left( \lim_{k \rightarrow\infty} \hat{p}_k(s'\mid h,a) - \bar{p}_k(s' \mid h,a) = 0 \right) = 1. \label{eq:SLLN_P}
\end{align}

This gives $\mathbb{P}(\lim_{k \rightarrow\infty} ||\hat{p}_k - \bar{p}_k||_\infty = 0) = 1$. To complete the proof of almost sure convergence, we show that the value function and $Q$ functions as defined previously are Lipschitz functions of the underlying transition probability over histories $P$, for a fixed reward function $R$. The max function in the computation of $V_k$ is $1$-Lipschitz in $Q$, and the expectation in the true $Q$ function is $|\mathcal{S}| V_{\max}$-Lipschitz in $P$ due to Assumption~\ref{as:bounded_v}. The Lipschitz constants here are with respect to the $\ell_\infty$ norm of the arguments. Thus every component of $V$ and $Q$ is a compositions of Lipschitz functions. Furthermore, it is easy to see that the Lipschitz constant from one level of the tree to the one above it is at most $|\mathcal{S}|V_{\max}$, making the Lipschitz constant at the root\footnote{And thus all the Lipschitz constants are uniformly bounded by $(|\mathcal{S}| V_{\max})^D$.} $(|\mathcal{S}|V_{\max})^D$. To summarize, we have
\begin{align}\label{eq:lipschitz}
||Q(\cdot, \cdot;P,R) - Q(\cdot, \cdot;P',R)||_\infty \leq (|\mathcal{S}|V_{\max})^D ||P - P'||_\infty
\end{align}
To finish the proof of almost sure convergence, we use the Lipschitz result~\eqref{eq:lipschitz} with the SLLN result~\eqref{eq:SLLN_P}.
\begin{align*}
||V_k -\Vstark||_\infty &\leq (|\mathcal{S}|V_{\max})^D ||\hat{p}_k - \bar{p}_k||_\infty.
\end{align*}
The right-hand side converges to zero almost surely, and the left-hand side is non-negative so it must also converge to zero almost surely.
\end{proof}

\subsection{Proof of Theorem 2}

Now, based on the above, we may prove Theorem \ref{thm:conv}.

\begin{proof}[Proof of Theorem \ref{thm:conv}] We will begin by rewriting Equation \ref{eqn:gwfp} in the specific notation of Algorithm \ref{alg:ramcp}. Let $b_k$ denote the average belief over model parameters after $k$ iterations.  Therefore, Equation \ref{eqn:gwfp} is equivalent to
\begin{equation}
    b_{k+1} = (1 - \alpha_{k+1}) b_k + \alpha_{k+1} BR_\epsilon(\pi_k)
\end{equation}
\begin{equation}
    \pi_{k+1} = (1 - \alpha_{k+1}) \pi_k + \alpha_{k+1} BR_\epsilon(b_k),
\end{equation}
where we set $M_k =0$ for all $k$. We will show that the belief and policy updates that are being performed in Algorithm \ref{alg:ramcp} satisfy the above. We will first look at the policy update. By Lemma \ref{lem:q}, the approximate values computed from the tree converge in probability to the optimal function for the averaged set of beliefs. This implies the policy is an $\epsilon$-best response, with $\epsilon \to 0$ as $k \to \infty$. Similarly, for the belief update, the value estimate converges in probability to the optimal value for each model. Thus the computed solution to Equation \ref{eqn:opt_metric} converges to the best response as $k\to \infty$. Therefore both the belief update and the policy update satisfy the definition of a GWFP.

By Lemma \ref{lem:conv}, any GWFP process will converge to the set of Nash Equilibria in zero-sum, two-player games. Note that since the infinite action set for the adversary, $\mathcal{B}$, is convex, any Nash equilibrium will correspond to a solution of Equation \ref{eqn:opt_objective} by the Minimax theorem. This implies that as $k \to \infty$, the computed belief and policy converge to the optimal Nash equilibrium, meaning $\lim_{k \to \infty} \pi_k \in \Pi^*$ in probability.
\end{proof}

\section{Weighted Tree Updates}
\label{app:weighted_tree}

When we sample $\theta$ at the root of the tree, and then follow any policy up to history $h$, the distribution of samples of $\theta$ at a node $h$ will be distributed according to $p(\theta_i | h)$. We refer to Lemma 1 in \cite{guez2013scalable} for a proof. We require that $\theta_i$ be sampled uniformly to get consistent estimates of $V_{\theta_i}(\pi)$ for every $\theta$. However, we would like to update the tree's estimates to reflect values corresponding to simulation with $\theta_i \sim b_j(\theta)$, the adversarially chosen distribution.

Let $Q_q(h,a)$ represent a sample of $Q(h,a)$ obtained at node $(h,a)$ where $\theta$ was sampled from from $q(\theta)$ at the root. The standard Monte Carlo updates for the estimator are:
\begin{align}
\begin{split} \label{eqn:normal_updates}
    \hat{N}^{(k+1)}_q(h,a) &= \hat{N}^{(k)}_q(h,a) + 1 \\
    \hat{Q}^{(k+1)}_q(h,a) &= \hat{Q}^{(k)}_q(h,a) + \frac{Q^k_q(h,a) - \hat{Q}^{(k)}_q(h,a)}{\hat{N}^{(k)}_q(h,a)}
\end{split}
\end{align}
with
\begin{align}
        \hat{N}^{(0)}_q(h,a) &= 0\\
        \hat{Q}^{(0)}_q(h,a) &= 0.
\end{align}
Define the weighted estimators $\hat{Q}^{(k)}_{q,w}(h,a)$ with the update equation as
\begin{align}
\begin{split} \label{eqn:weighted_updates}
\hat{N}^{(k+1)}_{q,w}(h,a) &= \hat{N}^{(k)}_{q,w}(h,a) + 1 \\
\hat{Q}^{(k+1)}_{q,w}(h,a) &= \hat{Q}^{(k)}_{q,w}(h,a) + \frac{ w(\theta_k) Q^k_q(h,a) - \hat{Q}^{(k)}_q(h,a)}{\hat{N}^{(k)}_q(h,a)}
\end{split}
\end{align}
with
\begin{align}
    \hat{N}^{(0)}_{q,w}(h,a) &= 0 \\
    \hat{Q}^{(0)}_{q,w}(h,a) &= 0.
\end{align}
With the correct choice of weighting, we can sample $\theta$ from one distribution $q$, while obtaining estimates of the Q values as if $\theta$ was sampled from a different distribution $p$.

\begin{theorem}[Consistency of Weighted Tree Updates]
If $w = p/q$, then $ \lim_{k\rightarrow\infty} \hat{Q}^{(k)}_{p}(h,a) = \lim_{k\rightarrow\infty} \hat{Q}^{(k)}_{q,w}(h,a)$, i.e. the estimators are consistent.
\end{theorem}

\begin{proof}
Note that the normal recurrence relation (\ref{eqn:normal_updates}) corresponds to the explicit formula
\[
\hat{Q}^{(k)}_p(h,a) = \frac{1}{N^{(k)}_p(h,a)} \sum_{i=1}^{N^{(k)}_p(h,a)} Q^i_p(h,a).
\]
This is a Monte Carlo estimate of $Q^i_p(h,a)$, and thus will converge to
\begin{align}
\small
    \E [ Q^i_p(h,a) ] &=  \sum_{i=1}^M \E [ Q^i_p(h,a) | \theta_i ]  p(\theta_i | h) \label{eqn:step1} \\
    &=  \sum_{i=1}^M \frac{ \E [ Q^i_p(h,a) | \theta_i ]  p(h | \theta_i) p(\theta_i) }{p(h)}, \label{eqn:step2}
\end{align}
where Equation \ref{eqn:step1} follows from the definition of the expectation, and Equation \ref{eqn:step2} was obtained by application of Bayes' rule.

Similarly, the weighted recurrence corresponds to the explicit formula
\[
\hat{Q}^{(k)}_{q,w}(h,a) = \frac{1}{N^{(k)}_{q,w}(h,a)} \sum_{i=1}^{N^{(k)}_{q,w}(h,a)} w(\theta_i) Q^i_q(h,a).
\]
This is a Monte Carlo estimate of $w(\theta_i) Q^i_q(h,a)$, which converges to the expected value of $w(\theta_i) Q^i_q(h,a)$:
\begin{align}
\small
    \E [ &w(\theta_i) Q^i_q(h,a) ] = \sum_{i=1}^M \E [ w(\theta_i) Q^i_q(h,a) | \theta_i ]  p(\theta_i | h)  \label{eqn:weighted_step1}\\
    &= \sum_{i=1}^M \frac{ \E [ Q^i_q(h,a) | \theta_i ] w(\theta_i)  p(h | \theta_i) q(\theta_i) }{p(h)} \label{eqn:weighted_step2} \\
    &= \sum_{i=1}^M \frac{ \E [ Q^i_q(h,a) | \theta_i ]  p(h | \theta_i) p(\theta_i) }{p(h)} \label{eqn:weighted_step3} \\
    &= \sum_{i=1}^M \frac{ \E [ Q^i_p(h,a) | \theta_i ]  p(h | \theta_i) p(\theta_i) }{p(h)} \label{eqn:weighted_step4} \\
    &= \E [ Q^i_p(h,a) ].
\end{align}
From Equation \ref{eqn:weighted_step1} to Equation \ref{eqn:weighted_step2}, we use the fact that $w(\theta_i)$ is constant within the conditional expectation. Equation \ref{eqn:weighted_step3} is obtained by substituting $w = q/p$. Equation \ref{eqn:weighted_step4} follows because conditioned on $\theta_i$, the estimates of $Q^i_p$ and $Q^i_q$ should have the same distribution, since the only difference between them is in the sampling of $\theta$. Since the two update formulas are Monte Carlo estimates of quantities with the same expectation, their value as $N(h,a) \rightarrow \infty$ will be the same. Since the number of times a node is simulated goes to infinity with the number of iterations, these estimators must converge to the same value in the limit as $k \rightarrow \infty$ as well.
\end{proof}

\section{Experimental Details}
\label{sec:bandit}

\subsection{Bandit Reward Model}

\begin{table}[h]
\footnotesize
  \caption{\label{tab:mab} Bandit model for the $n$-pull bandit experiment. Each cell lists the probability of obtaining the reward $R_i$ under one of the two models, where the reward $R_i$ is listed in the column heading. Each column corresponds to a certain reward, each row corresponds to a certain action/model combination. }
  \centering
  \begin{tabular}{ | c | c | c | c | c | c | c |}
    \hline
    $R$ & $-1.0$ & $-0.5$ & $-0.1$ & $ 0.0$ & $0.5$ & $1.0$ \\ \hline \hline

    $P(R | a_1,\theta_1)$ & 0.0 & 0.0 & 1.0 & 0.0 & 0.0 & 0.0 \\ \hline
    $P(R | a_2,\theta_1)$ & 0.0 & 0.0 & 0.0 & 0.0 & 1.0 & 0.0\\ \hline
    $P(R | a_3,\theta_1)$ & 0.2 & 0.0 & 0.0 & 0.0 & 0.0 & 0.8\\ \hline
    $P(R | a_4,\theta_1)$ & 0.8 & 0.0 & 0.0 & 0.0 & 0.0 & 0.2\\ \hline \hline

    $P(R | a_1,\theta_2)$ & 0.0 & 0.0 & 0.0 & 1.0 & 0.0 & 0.0\\ \hline
    $P(R | a_2,\theta_2)$ & 0.0 & 1.0 & 0.0 & 0.0 & 0.0 & 0.0\\ \hline
    $P(R | a_3,\theta_2)$ & 0.8 & 0.0 & 0.0 & 0.0 & 0.0 & 0.2\\ \hline
    $P(R | a_4,\theta_2)$ & 0.2 & 0.0 & 0.0 & 0.0 & 0.0 & 0.8\\ \hline
  \end{tabular}
\end{table}

There are several features to note about this bandit problem. First, the rewards for $a_1$ and $a_2$ are not stochastic within each model. The important result of this is that taking one of these actions completely disambiguates between the models. For example, if an agent takes action $a_1$ and receives a reward of $-0.1$, the posterior probability of $\theta_1$ is $1$, and the probability of $\theta_2$ is $0$. This is important, as on an $n$-pull bandit problem, an agent action optimally in a Bayes-adaptive sense will aim to trade off exploration and exploitation. The difference between the two actions is the magnitudes of the reward for each model. If an agent is acting in a risk-sensitive fashion with respect to models, given for example a prior belief over models of $0.5$ for each model, the agent may prefer $a_1$, which has a lower expected reward but also a lower variance. In the risk-neutral case, a Bayes-adaptive agent will prefer $a_2$, which has higher expected reward but also higher variance, and lower worst-model performance.

For actions $a_3$ and $a_4$, the reward distributions will have the same mean, variance, and worst-model performance for a uniform prior over models. However, given better knowledge of the model, the expected value of one of the two actions increases. Therefore, in the case where the agent takes action $a_1$ or $a_2$ to disambiguate between the models (thus setting the posterior probability of one of the models to $1$), the expected reward under either $a_3$ or $a_4$ will be 0.6. Therefore, in this environment, a Bayes-adaptive optimal agent will disambiguate between models by taking action $a_1$ or $a_2$, and then exploit by taking either $a_3$ or $a_4$. The specific actions and the associated value of the optimal Bayes-adaptive policy depend on the risk metric chosen.

\subsection{Patient Treatment Model}
The transition models used in the patient treatment experiment were generated as follows. For each action, we allowed relative changes in state $s' - s = \Delta s \in \{-3,-2,-1,0,1,2,3\}$. A $3 \times 7$ matrix representing the likelihood of these 7 possible outcomes conditioned on the three possible actions was created randomly, once for each of the 15 possible underlying models $\theta$. For each $\theta$, the transition matrix was created by sampling 3 rows of a $7\times7$ identity matrix, adding normally distributed zero-mean noise with standard deviation $0.1$, and then normalizing the rows to sum to one. The specific probabilities used in the experiments presented are presented below.

The reward function was defined as $R(s,a,s') = s'/20 - 2\cdot\mathds{1}_{s'=0}$, giving increasing rewards as health increases, with a -2 penalty for death ($s=0$).

\section{Further Experimental Results}
\label{sec:further}

\subsection{Multi-armed Bandit}

Figure \ref{fig:bandit_stats} shows 95\% confidence intervals for \algName{} for varying CVaR risk levels. Note that as expected, lower values of $\alpha$ result in policies with substantially lower variance over realized cost. However, this comes at the cost of a slight degredation in mean accrued reward. The non Bayes-adaptive policy (which is Markovian) is also plotted. Because it does not incorporate the information gained from previous transitions, it does not actively disambiguate between models and thus results in low mean reward.

\onecolumn
\begin{figure}[t!]
\centering
    \includegraphics[width=0.66\linewidth]{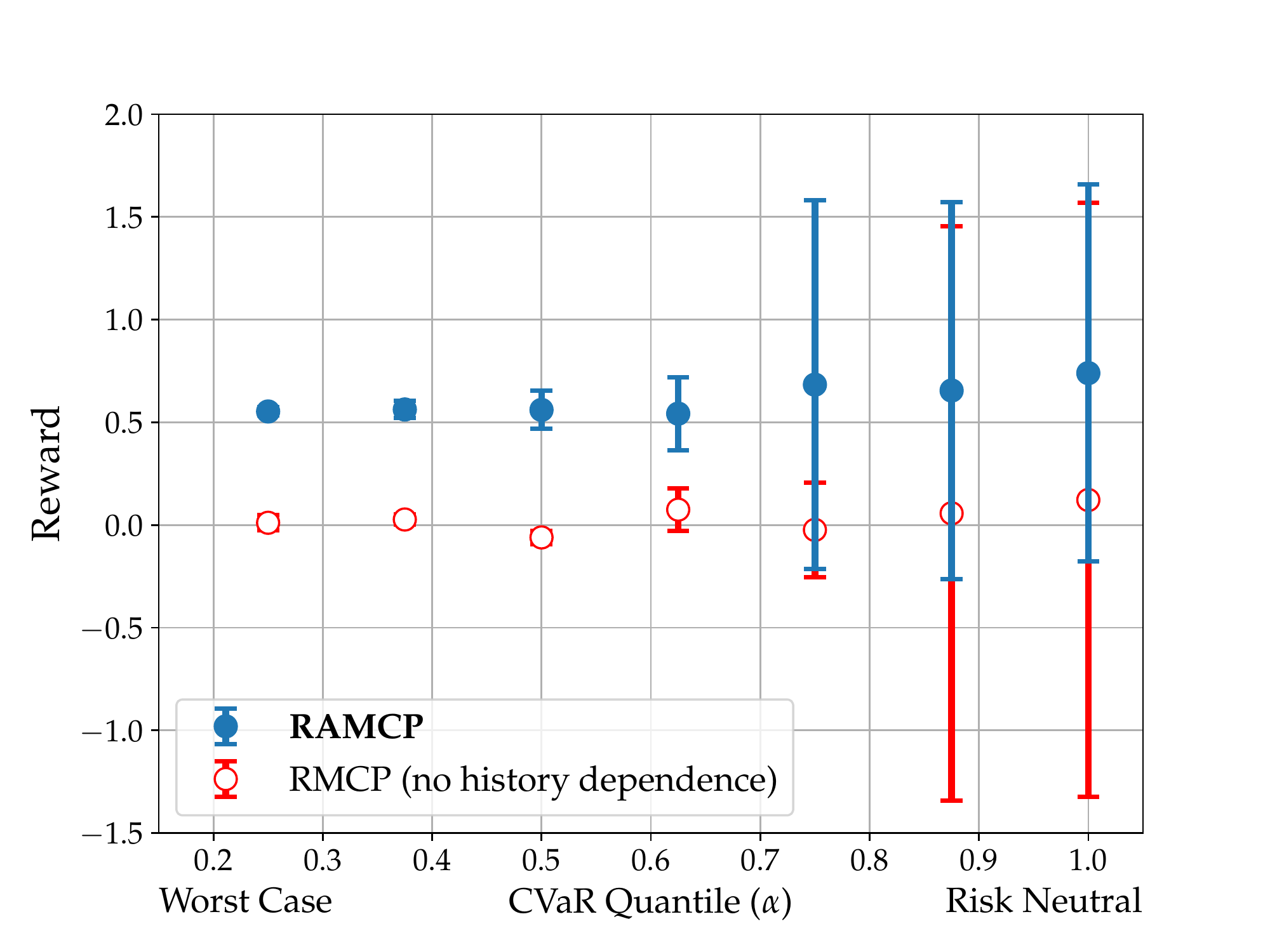}
    \caption{Performance of \algName{} on a bandit problem for varying CVaR quantiles ($\alpha$). For each $\alpha$, 500 trials were performed. In this plot, the error bars denote the 95\% confidence bound corresponding to variation in performance due to uncertainty over the underlying model. Applying this same risk-sensitive optimization method on a state dependent (RMCP) policy which performs consistently worse than the history dependent RAMCP policy, highlighting the importance of planning with adaptation in mind.} \label{fig:bandit_stats}
\end{figure}

\end{document}